\documentclass[runningheads]{llncs}

\usepackage[T1]{fontenc}
\usepackage{mathtools,amssymb}
\usepackage{hyperref}
\usepackage{tikz}
\usetikzlibrary{arrows.meta,positioning,fit}

\hypersetup{
  hidelinks,
  pdftitle={Graph Surgery and the Do-Operator},
  pdfauthor={Satpreet Makhija},
  pdfsubject={A precise correspondence for acyclic structural causal models},
  pdfkeywords={structural causal models, interventions, do-operator,
    graph surgery, dependency graphs}
}

\allowdisplaybreaks[1]

\DeclareMathOperator{\An}{An}
\DeclareMathOperator{\Run}{Run}
\DeclareMathOperator{\Surg}{Surg}
\DeclareMathOperator{\Do}{Do}
\DeclareMathOperator{\Graph}{Graph}
\DeclareMathOperator{\doexpr}{do}

\title{Graph Surgery and the Do-Operator}
\subtitle{A Precise Correspondence for Acyclic Structural Causal Models}
\titlerunning{Graph Surgery and the Do-Operator}
\author{Satpreet Makhija\inst{1}}
\authorrunning{S. Makhija}
\institute{Ashoka University, Sonipat, India\\
\email{satpreet.makhija@ashoka.edu.in}}

\begin{document}

\maketitle

\begin{abstract}
The \(\doexpr\)-operator is described graphically by deleting arrows into its
targets and functionally by replacing their mechanisms with constants.  To
call these operations equivalent is not yet a mathematical statement: one
returns a graph and remembers only the targets, whereas the other returns
mechanisms and also remembers the imposed values.  We make a dependency-level
comparison precise for deterministic acyclic structural causal models with
finitely many endogenous variables.  If \(\Graph(F)\) extracts the
dependencies of a mechanism family \(F\), our main theorem is
\[
  \Graph(F^\iota)
  =
  \Surg\bigl(\Graph(F),T_\iota\bigr).
\]
Thus replacing target mechanisms removes exactly the dependencies removed by
graph surgery.  For a model \(M=(G,F)\) whose graph may contain unused arrows,
we characterize when the same equality holds with \(G\) in place of
\(\Graph(F)\); it holds for every intervention exactly when \(G\) records the
dependencies of \(F\) exactly.  We then define the intervened model,
characterize its run, show how sequential interventions combine, and prove
that an outcome depends only on interventions at its actual dependency
ancestors.  All principal results are machine-checked in an accompanying
Lean~4 development.
\end{abstract}

\section{Introduction}
\label{sec:introduction}

The \(\doexpr\)-operator is the standard notation for an intervention in a
structural causal model~\cite{pearl2009causality}.  An expression such as
\(\doexpr(A=a)\) says that \(A\) is no longer computed by its ordinary mechanism:
it is supplied by the experimenter and held at \(a\).  A formal semantics
should identify the resulting model and explain how it behaves.

Zhang conjectured that deleting incoming arrows and replacing mechanisms by
constants are equivalent views of intervention~\cite[Sec.~6.4]{zhang2025thesis}.
The word ``equivalent'' needs to be made precise.  Graph surgery and mechanism
replacement are not literally the same operation: they act on different
objects and, more importantly, graph surgery cannot distinguish
\(\doexpr(A=0)\) from \(\doexpr(A=1)\).  A natural comparison is between the
dependencies that remain.  We formulate that comparison and prove it for
deterministic acyclic structural causal models with finitely many endogenous
variables.

Our starting point is a causal model \(M\coloneqq(G,F)\) with two components.  The
directed acyclic graph \(G\) describes which variables may directly affect
which others, and the mechanism family \(F\) says how their values are
computed.  Compatibility
requires every dependency used by \(F\) to occur as an arrow of \(G\); it does
not require every arrow to be used.  We call the model exact when the arrows of
\(G\) are precisely the dependencies of \(F\).  For an exogenous state \(u\),
\(\Run(M,u)\) is the unique world generated by the mechanisms.

We represent a simultaneous intervention by a type-respecting partial map
\(\iota\) on \(V\), whose domain \(T_\iota\) is its target set.
On the graph, \(\Surg(G,T_\iota)\) deletes the arrows entering the targets.
In the equations, \(F^\iota\) replaces the corresponding
mechanisms with the constants supplied by \(\iota\).  Our central result says
that the same dependency graph is reached in either order: update the
mechanisms and then extract their dependencies, or first extract the
dependencies and then perform graph surgery.  When the supplied graph \(G\)
is exact, this is also the graph obtained by performing surgery directly on
\(G\).

Our main results are:
\begin{enumerate}
  \item an exact correspondence between graph surgery and constant mechanism
        replacement (Theorem~\ref{thm:surgery-do});
  \item a characterization of when this correspondence agrees with the graph
        supplied as part of a causal model
        (Corollary~\ref{cor:supplied-graph});
  \item a law for combining sequential interventions, including the fact that
        a later assignment replaces an earlier one
        (Theorem~\ref{thm:composition}); and
  \item an ancestor theorem: the value of a set of variables depends only on
        the part of the intervention applied to their actual dependency
        ancestors
        (Theorem~\ref{thm:locality}).
\end{enumerate}

The principal results are also mechanized in Lean~4.  The accompanying
artifact checks the headline correspondence at its full generality, before
any finiteness or acyclicity assumptions are introduced, and then verifies
the model-level, evaluation, composition, and locality results.

\section{Structural causal models}
\label{sec:models}

Let \(V\) be a finite set of endogenous variables.  Each \(v\in V\) has a
nonempty value set \(\mathcal X_v\).  Let
\[
  \mathcal X\coloneqq\prod_{v\in V}\mathcal X_v
\]
be the set of complete assignments, which we call \emph{worlds}.  For
\(S\subseteq V\), write \(x_S\) for the restriction of a world
\(x\in\mathcal X\) to \(S\).  Let \(\mathcal U\) be a nonempty set of
exogenous states.  An element
\(u\in\mathcal U\) may collect local disturbances, shared background factors,
or any other information fixed outside the endogenous equations.  No
probability law is needed.

A mechanism family on \(V\) is a collection of functions
\[
  F_v:\mathcal U\times\mathcal X\longrightarrow\mathcal X_v
  \qquad(v\in V).
\]

We give every mechanism the common domain
\(\mathcal U\times\mathcal X\) deliberately.  This avoids building an
a priori parent set into the type of \(F_v\) and lets the dependencies be
recovered from the functions themselves: a coordinate that \(F_v\) ignores
contributes no arrow.  Compatibility will require every coordinate that can
affect \(F_v\) to be permitted by \(G\), while acyclicity will make the
resulting equations evaluable in topological order.

\begin{definition}[Dependency graph]
\label{def:dependency-graph}
A mechanism \(F_v\) \emph{depends on} \(w\in V\) if changing only \(w\) can
change the value returned by \(F_v\).  That is, there are
\(u\in\mathcal U\) and worlds \(x,y\in\mathcal X\) that agree at every
variable except possibly \(w\), but for which
\(F_v(u,x)\ne F_v(u,y)\).  The dependency graph of \(F\), written
\(\Graph(F)\), has vertex set \(V\) and an arrow \(w\to v\) exactly when
\(F_v\) depends on \(w\).
\end{definition}

Thus the graph records actual rather than merely permitted dependencies.

\begin{definition}[Causal model]
\label{def:model}
A deterministic acyclic structural causal model is a pair \(M\coloneqq(G,F)\), where
\(G\) is a directed acyclic graph and \(F\) is a mechanism family compatible
with it: every arrow of \(\Graph(F)\) is also an arrow of \(G\).  A variable
\(w\) is a parent of \(v\) when \(G\) has an arrow \(w\to v\).  The model is
\emph{exact} when \(G=\Graph(F)\).
\end{definition}

Compatibility gives the graph and mechanisms distinct roles.  The graph says
which dependencies are permitted; the mechanisms determine which of them are
used in a particular model.  For example, \(G\) may contain \(w\to v\) even
when \(F_v\) ignores \(w\); such a model is compatible but not exact.  Although
\(F_v\) is written with a whole world as input, compatibility ensures that its
value depends only on the parents of \(v\).

\begin{lemma}[Only parents matter]
\label{lem:parent-sufficiency}
If \(M=(G,F)\) is a causal model and two worlds \(x,y\in\mathcal X\) agree on
every parent of \(v\), then
\[
  F_v(u,x)=F_v(u,y)
\]
for every \(u\in\mathcal U\).
\end{lemma}

\begin{proof}
The set \(V\) is finite.  Pass from \(x\) to \(y\) by changing one variable at
a time.  Only nonparents need to be changed.  If \(w\) is not a parent of
\(v\), compatibility implies that \(w\to v\) is not an arrow of
\(\Graph(F)\).  By Definition~\ref{def:dependency-graph}, changing \(w\) alone
therefore leaves the value returned by \(F_v\) unchanged.
\end{proof}

\begin{lemma}[Unique evaluation]
\label{lem:unique-evaluation}
For every causal model \(M=(G,F)\) and exogenous state \(u\in\mathcal U\),
there is exactly one world \(x\in\mathcal X\) satisfying
\[
  x_v=F_v(u,x)
  \qquad(v\in V).
\]
\end{lemma}

\begin{proof}
Order the variables as \(v_1,\ldots,v_n\) so that every parent comes before its
children.  Construct a world in that order.  At step \(k\), fill the unassigned
coordinates arbitrarily and evaluate \(F_{v_k}\).  Its parents already have
values, so Lemma~\ref{lem:parent-sufficiency} makes the result independent of
those temporary choices.  Use it as \(x_{v_k}\).  The completed world satisfies
every equation.

If \(x\) and \(y\) both satisfy the equations, induction along the same order
shows that they agree.  Indeed, they agree at the parents of \(v_k\), so
Lemma~\ref{lem:parent-sufficiency} gives
\(F_{v_k}(u,x)=F_{v_k}(u,y)\), and hence \(x_{v_k}=y_{v_k}\).
\end{proof}

\begin{definition}[Run]
\label{def:run}
For a model \(M=(G,F)\) and exogenous state \(u\in\mathcal U\), define
\(\Run(M,u)\) to be the unique world \(x\in\mathcal X\) such that
\[
  x_v=F_v(u,x)
  \qquad\text{for every }v\in V.
\]
\end{definition}

\section{The \texorpdfstring{\(\doexpr\)}{do}-operator}
\label{sec:intervention}

An intervention supplies values at some variables and leaves all others
unspecified.  We represent it by listing values only for the variables it
targets.

\begin{definition}[Intervention]
\label{def:intervention}
An intervention \(\iota\) is a type-respecting partial map on \(V\).  Its
domain \(T_\iota\coloneqq\operatorname{dom}(\iota)\subseteq V\) is the target
set, and \(\iota(v)\in\mathcal X_v\) for every \(v\in T_\iota\).  We write
\(\iota_v\coloneqq\iota(v)\).  The empty intervention is the empty map, written
\(\varnothing\).  For \(S\subseteq V\), \(\iota|_S\) denotes the ordinary
restriction of \(\iota\) to \(S\).
When \(T_\iota=\{A\}\) and \(\iota_A=a\), we write the intervention as
\(\doexpr(A=a)\).
\end{definition}

\begin{definition}[Graph surgery]
\label{def:surgery}
For a directed graph \(G\) on \(V\) and a target set \(T\subseteq V\),
\(\Surg(G,T)\) has the same vertices as \(G\) and deletes exactly the arrows
entering vertices in \(T\).
\end{definition}

Graph surgery depends only on the target set and does not record the values
imposed there.  Those values enter through the following change to the
mechanisms.

\begin{definition}[Mechanism replacement]
\label{def:replacement}
For a mechanism family \(F\) on \(V\) and intervention \(\iota\), define
\(F^\iota\) by
\[
  F^\iota_v(u,x)\coloneqq
  \begin{cases}
    \iota_v, & v\in T_\iota,\\
    F_v(u,x), & v\notin T_\iota.
  \end{cases}
\]
\end{definition}

\begin{example}[A three-variable intervention]
\label{ex:basic}
Let \(V\coloneqq\{A,B,C\}\), let every endogenous value set be
\(\mathbb R\), and let \(\mathcal U\coloneqq\mathbb R^2\).  For
\(u=(u_A,u_B)\), define
\[
  F_A(u,x)\coloneqq u_A,
  \qquad
  F_B(u,x)\coloneqq x_A+u_B,
  \qquad
  F_C(u,x)\coloneqq 2x_B.
\]
Thus \(\Graph(F)\) is the chain
\[
  A\longrightarrow B\longrightarrow C.
\]
If \(\iota\coloneqq\doexpr(B=7)\), then \(F^\iota_B\) is the constant \(7\), while
the other two mechanisms are unchanged.  Hence \(\Graph(F^\iota)\) consists
only of \(B\to C\), exactly the graph obtained by deleting the arrow entering
\(B\).  The updated equations have the unique solution
\[
  (A,B,C)=(u_A,7,14).
\]
\end{example}

\begin{theorem}[Graph--mechanism correspondence]
\label{thm:surgery-do}
For every mechanism family \(F\) on \(V\) and intervention \(\iota\),
\[
  \Graph(F^\iota)
  =
  \Surg\bigl(\Graph(F),T_\iota\bigr).
\]
\end{theorem}

\begin{proof}
Fix two variables \(w,v\in V\).  If \(v\) is a target, then
\(F^\iota_v\) is constant, so no arrow enters \(v\) in
\(\Graph(F^\iota)\); graph surgery likewise deletes every arrow entering
\(v\).  If \(v\) is not a target, then \(F^\iota_v=F_v\), so
\(w\to v\) is an arrow of \(\Graph(F^\iota)\) exactly when it is an arrow of
\(\Graph(F)\); graph surgery leaves all such arrows unchanged.  The two graphs
therefore have the same arrows.
\end{proof}

The identity compares the graphical and functional descriptions at the level
of dependencies.  A separate question is whether surgery on a graph \(G\)
supplied with the model gives that same graph.

\begin{corollary}[Agreement with the supplied graph]
\label{cor:supplied-graph}
Let \(M=(G,F)\) be a causal model and let \(\iota\) be an intervention.
Then \(F^\iota\) is compatible with \(\Surg(G,T_\iota)\).  Moreover,
\[
  \Graph(F^\iota)=\Surg(G,T_\iota)
\]
if and only if every arrow of \(G\) that is absent from \(\Graph(F)\) enters
an intervention target.  Consequently, the equality holds for every
intervention if and only if \(M\) is exact.
\end{corollary}

\begin{proof}
Compatibility says that \(\Graph(F)\) is a subgraph of \(G\).  Applying the
same surgery to both graphs preserves this relation, and
Theorem~\ref{thm:surgery-do} identifies the smaller graph with
\(\Graph(F^\iota)\).  This proves compatibility after intervention.

The two surgically modified graphs are equal exactly when no extra arrow
survives.  An arrow survives surgery exactly when it does not enter a target.
Thus equality holds exactly when every arrow of \(G\) absent from
\(\Graph(F)\) enters a target.  The uniform statement follows: exact models
have no extra arrows, while equality for the empty intervention gives
\(\Graph(F)=G\).
\end{proof}

\begin{example}[Why the supplied graph may differ]
\label{ex:inexact}
Let \(V\coloneqq\{A,B\}\), let the only arrow of \(G\) be \(A\to B\), and suppose
neither mechanism depends on an endogenous variable.  Then \(\Graph(F)\) has
no arrows, so \((G,F)\) is compatible but not exact.  An intervention that
does not target \(B\) leaves \(A\to B\) in \(\Surg(G,T_\iota)\), whereas
\(\Graph(F^\iota)\) has no arrows.  If the intervention targets \(B\), both
graphs have no arrows and therefore agree.  Thus exactness is sufficient for
agreement under every intervention, but it is not necessary for agreement
under a particular one.
\end{example}

Having related the two separate edits, we now package their results as a
complete intervened model.

\begin{definition}[Do-operation]
\label{def:do}
For a causal model \(M=(G,F)\) and intervention \(\iota\), define
\[
  \Do(M,\iota)
  \coloneqq
  \bigl(\Surg(G,T_\iota),F^\iota\bigr).
\]
This pair is a causal model: deleting arrows preserves acyclicity, and
Corollary~\ref{cor:supplied-graph} supplies compatibility.
\end{definition}

The do-operation applies one edit to each component of the model.  An
alternative construction first replaces the mechanisms and then equips them
with their exact dependency graph, giving
\(\bigl(\Graph(F^\iota),F^\iota\bigr)\).

\begin{corollary}[Agreement of the two constructions]
\label{cor:two-routes}
For every causal model \(M=(G,F)\) and intervention \(\iota\), the pair
\(\bigl(\Graph(F^\iota),F^\iota\bigr)\) is an exact causal model.  For every
exogenous state \(u\),
\[
  \Run(\Do(M,\iota),u)
  =
  \Run\bigl((\Graph(F^\iota),F^\iota),u\bigr).
\]
If \(M\) is exact, the stronger model equality also holds:
\[
  \Do(M,\iota)
  =
  \bigl(\Graph(F^\iota),F^\iota\bigr).
\]
\end{corollary}

\begin{proof}
Corollary~\ref{cor:supplied-graph} says that \(\Graph(F^\iota)\) is a subgraph
of the acyclic graph \(\Surg(G,T_\iota)\), so it is acyclic.  The displayed
pair is therefore a causal model, and it is exact by construction.  Both
models use the mechanism family \(F^\iota\), which determines their runs.  If
\(M\) is exact, Corollary~\ref{cor:supplied-graph} also gives
\(\Surg(G,T_\iota)=\Graph(F^\iota)\), so the graphs, and hence the models, are
equal.
\end{proof}

\begin{figure}[t]
\centering
\begin{tikzpicture}[
  every node/.style={font=\small},
  map/.style={-{Latex[length=1.8mm]},line width=0.55pt},
  lab/.style={font=\scriptsize,fill=white,inner sep=2pt,text=black}
]
  \node (f)  at (0,0)       {\(F\)};
  \node (fi) at (5.2,0)     {\(F^\iota\)};
  \node (g)  at (0,-1.8)    {\(\Graph(F)\)};
  \node (gi) at (5.2,-1.8)  {\(\Graph(F^\iota)\)};

  \draw[map] (f)--node[lab,above]{replace target mechanisms} (fi);
  \draw[map] (f)--node[lab,left]{extract dependencies} (g);
  \draw[map] (fi)--node[lab,right]{extract dependencies} (gi);
  \draw[map] (g)--node[lab,below]{surgery at \(T_\iota\)} (gi);
\end{tikzpicture}
\caption{The correspondence of Theorem~\ref{thm:surgery-do}.  Replacing target
mechanisms and extracting dependencies yields the same graph as extracting
dependencies first and then performing surgery.}
\label{fig:basic}
\end{figure}
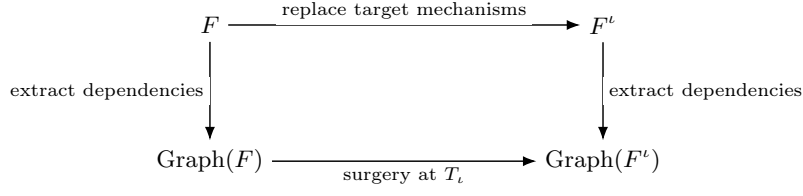

The intervened outcome is characterized directly by the following equations.

\begin{corollary}[Intervention equation]
\label{cor:intervention-equation}
For every model \(M=(G,F)\), intervention \(\iota\), exogenous state
\(u\in\mathcal U\), and world \(x\in\mathcal X\),
\[
  x=\Run(\Do(M,\iota),u)
  \quad\Longleftrightarrow\quad
  \begin{cases}
    x_v=\iota_v, & v\in T_\iota,\\
    x_v=F_v(u,x), & v\notin T_\iota.
  \end{cases}
\]
\end{corollary}

\begin{proof}
By Definition~\ref{def:run}, \(x\) is the run exactly when
\(x_v=F^\iota_v(u,x)\) for every \(v\).  At a target, this says
\(x_v=\iota_v\); at every other variable, it says \(x_v=F_v(u,x)\).
These are precisely the two displayed conditions.
\end{proof}

The corollary gives the complete test for an intervened outcome: each target
has its assigned value, and every equation outside the target set remains
unchanged.  Because \(G\) is acyclic, these conditions determine exactly one
world.

\section{Sequential interventions}
\label{sec:composition}

Suppose \(\iota\) is applied first and \(\kappa\) second.  If both assign a
value to the same variable, the later value from \(\kappa\) must prevail.

\begin{theorem}[Sequential interventions]
\label{thm:composition}
For interventions \(\iota\) and \(\kappa\), define \(\lambda\) to be the partial
map with domain \(T_\iota\cup T_\kappa\) that agrees with \(\kappa\) wherever
\(\kappa\) is defined and otherwise with \(\iota\).  Then, for every model
\(M\),
\[
  \Do\bigl(\Do(M,\iota),\kappa\bigr)
  =
  \Do(M,\lambda).
\]
\end{theorem}

\begin{proof}
On both sides, the graph is obtained from \(G\) by deleting the arrows entering
every variable targeted by either intervention.  At a target of \(\kappa\),
the mechanism is the constant \(\kappa_v\).  At a target of \(\iota\) but not
\(\kappa\), the constant \(\iota_v\) remains.  Every other mechanism is
\(F_v\).  Thus both the graphs and the mechanism families are equal.
\end{proof}

\begin{corollary}[Basic intervention laws]
\label{cor:laws}
For every model \(M\) and interventions \(\iota,\kappa\):
\begin{enumerate}
  \item The empty intervention does nothing:
    \(\Do(M,\varnothing)=M\).
  \item Repeating an intervention changes nothing:
    \(\Do(\Do(M,\iota),\iota)=\Do(M,\iota)\).
  \item Interventions with no common target may be applied in either order:
    \[
      \Do(\Do(M,\iota),\kappa)=\Do(\Do(M,\kappa),\iota).
    \]
  \item If both interventions target \(v\), the later value wins: the mechanism
    at \(v\) in \(\Do(\Do(M,\iota),\kappa)\) is the constant \(\kappa_v\).
\end{enumerate}
\end{corollary}

\begin{proof}
The first statement follows directly from Definition~\ref{def:do}, since
surgery with an empty target set leaves \(G\) unchanged and
\(F^\varnothing=F\).  The remaining statements follow from
Theorem~\ref{thm:composition} by checking the value retained at each target.
\end{proof}

For three or more interventions, the placement of parentheses does not matter:
applying any finite sequence gives the same model as one intervention that
collects all targets and keeps the last value assigned to each one.  This is
equality of the resulting models, not merely of their runs.

\section{Dependence on ancestors}
\label{sec:locality}

For an outcome set \(O\subseteq V\), write \(\An_F(O)\) for the variables in
\(O\) and every variable from which a directed path in the actual dependency
graph \(\Graph(F)\) reaches a variable in \(O\).  If \(v\in\An_F(O)\), then
every parent of \(v\) in \(\Graph(F)\) also lies in \(\An_F(O)\).
Consequently, the equation for any variable in this set depends only on
variables in the same set.

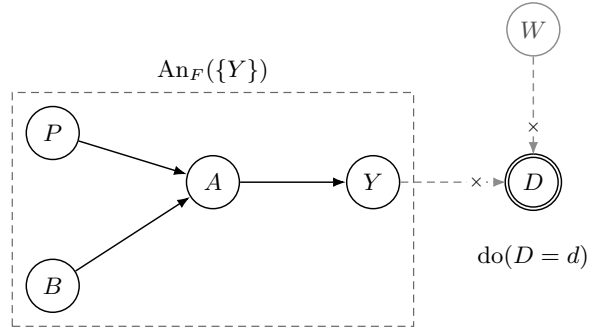
\begin{figure}[htbp]
\centering
\begin{tikzpicture}[
  node distance=13mm and 14mm,
  var/.style={circle,draw=black,line width=0.55pt,minimum size=7mm,
    inner sep=0pt,font=\small},
  target/.style={var,double,double distance=0.7pt},
  outside/.style={var,draw=black!45,text=black!70},
  edge/.style={-{Latex[length=1.7mm]},line width=0.55pt},
  removed/.style={edge,draw=black!45,densely dashed},
  cone/.style={draw=black!60,densely dashed,line width=0.45pt,inner sep=5pt}
]
  \node[var] (p) {\(P\)};
  \node[var,below=of p] (b) {\(B\)};
  \node[var,right=of p,yshift=-6.5mm] (a) {\(A\)};
  \node[var,right=of a] (y) {\(Y\)};
  \node[target,right=of y] (d) {\(D\)};
  \node[outside,above=of d] (w) {\(W\)};

  \draw[edge] (p)--(a);
  \draw[edge] (b)--(a);
  \draw[edge] (a)--(y);
  \draw[removed] (y)--node[pos=0.72,fill=white,inner sep=0.5pt]
    {\(\scriptstyle\times\)} (d);
  \draw[removed] (w)--node[pos=0.72,fill=white,inner sep=0.5pt]
    {\(\scriptstyle\times\)} (d);

  \node[cone,fit=(p)(b)(a)(y),
        label={[font=\footnotesize]above:\(\An_F(\{Y\})\)}] {};
  \node[font=\footnotesize,below=3.5mm of d] {\(\doexpr(D=d)\)};
\end{tikzpicture}
\caption{The outcome at \(Y\) is determined inside its ancestor set.  The
dashed, crossed arrows are removed by \(\doexpr(D=d)\); intervening on the
downstream variable \(D\) cannot change \(Y\).}
\label{fig:locality}
\end{figure}

\begin{theorem}[Dependence on ancestors]
\label{thm:locality}
Let \(M=(G,F)\) be a causal model and let \(O\subseteq V\).  If two
interventions \(\iota\) and \(\kappa\) agree on \(\An_F(O)\), so that
\(\iota|_{\An_F(O)}=\kappa|_{\An_F(O)}\), then for every \(u\in\mathcal U\),
\[
  \Run(\Do(M,\iota),u)_O
  =
  \Run(\Do(M,\kappa),u)_O.
\]
\end{theorem}

\begin{proof}
Order the variables topologically in \(\Graph(F)\).  We show by induction along
this order that the two runs agree at every \(v\in\An_F(O)\).

If \(v\) is targeted by either intervention, then
the assumption says it is targeted by both with the same value.  The two runs
agree at \(v\).  Otherwise both runs use the original mechanism \(F_v\).  By
the definition of \(\An_F(O)\), every parent of \(v\) in \(\Graph(F)\) lies in
\(\An_F(O)\), and all such parents precede \(v\).  The induction hypothesis
gives equal values at these parents.  Since \((\Graph(F),F)\) is an exact
causal model, Lemma~\ref{lem:parent-sufficiency} applied to that model gives
equal values at \(v\).  Thus the runs agree throughout \(\An_F(O)\), and in
particular on \(O\subseteq\An_F(O)\).
\end{proof}

\begin{corollary}[Discarding targets outside the ancestors]
\label{cor:restriction}
For every model \(M=(G,F)\), outcome set \(O\subseteq V\), intervention
\(\iota\), and exogenous state \(u\in\mathcal U\),
\[
  \Run(\Do(M,\iota),u)_O
  =
  \Run\bigl(\Do(M,\iota|_{\An_F(O)}),u\bigr)_O.
\]
In particular, if \(\iota\) has no target in \(\An_F(O)\), then
\[
  \Run(\Do(M,\iota),u)_O=\Run(M,u)_O.
\]
\end{corollary}

\begin{proof}
The interventions \(\iota\) and \(\iota|_{\An_F(O)}\) agree on \(\An_F(O)\).
Apply Theorem~\ref{thm:locality}.  If \(\iota\) has no target in
\(\An_F(O)\), then \(\iota|_{\An_F(O)}=\varnothing\), and the second equality
follows from the first law in Corollary~\ref{cor:laws}.
\end{proof}

Thus, to determine the outcome on \(O\), one may discard all intervention
targets outside \(\An_F(O)\).

\paragraph{The condition is sufficient, not necessary.}
Cancellation can make interventions on an actual ancestor irrelevant to an
outcome.  Take \(V=\{A,B,Y\}\) and
\(\mathcal U=\mathcal X_A=\mathcal X_B=\mathcal X_Y=\{0,1\}\), and define
\(F_A(u,x)=u\), \(F_B(u,x)=x_A\), and
\(F_Y(u,x)=x_A\mathbin{\oplus}x_B\).  The model
\((\Graph(F),F)\) is exact and \(A\) is an ancestor of \(Y\), but under
\(\doexpr(A=a)\) we have \(B=a\) and hence \(Y=a\oplus a=0\).  Thus
\(\doexpr(A=0)\) and \(\doexpr(A=1)\) disagree on an ancestor while producing
the same value of \(Y\).

\section{Lean mechanization}
\label{sec:mechanization}

The principal results of this paper have been mechanized in Lean~4
\cite{demoura2021lean4}.  The central graph--mechanism correspondence is the
declaration \path{dependencyGraph_replace}.  Its formal statement requires
neither a finite variable type nor acyclicity: those assumptions enter only in
the subsequent causal-model and evaluation layers.  The artifact additionally
checks compatibility after intervention, the characterization for supplied
graphs, uniqueness and equations for runs, sequential composition, agreement
of the two intervention routes, and actual-ancestor locality.

The development pins the Lean compiler and every external dependency.  Its
verification script rebuilds the formalization, prints the axiom footprint of
each headline theorem, and rejects proof placeholders or custom axioms in the
do-semantics source.  The complete source, theorem-to-code cross-reference,
and reproducibility instructions are available at
\url{https://github.com/SatpreetMakhija/graph-surgery-do-operator}.

\section{Related work}
\label{sec:related}

The \(\doexpr\)-operator is central to Pearl's account of structural causal
models~\cite{pearl2009causality}.  In the usual structural-equation account,
an intervention replaces selected equations by assigned values; deleting
their incoming arrows is the corresponding graphical operation.  Peters,
Janzing, and Sch{\"o}lkopf give a modern treatment of functional models,
interventions, and causal graphs~\cite{peters2017elements}, while Halpern gives
axioms for reasoning with interventions in recursive
models~\cite{halpern2000axiomatizing}.

String-diagram accounts give explicit graphical languages for causal models.
Jacobs, Kissinger, and Zanasi make the syntax--semantics separation formal:
string diagrams are interpreted as stochastic matrices, and intervention is
an operation on the diagrammatic syntax~\cite{jacobs2021surgery}.  Lorenz and
Tull develop the approach for a broader class of causal
models~\cite{lorenz2023causal}.  Our result addresses a different, narrower
question for ordinary deterministic structural causal models: how extensional
dependencies change under constant mechanism replacement, including when a
graph supplied with the mechanisms contains unused arrows.

Recent work uses mechanism functions to explain \(d\)-separation
semantically~\cite{zhang2026semantic}.  Zhang's earlier thesis describes
incoming-edge deletion as the syntactic view of an intervention, constant
replacement as its semantic view, and conjectures that the two are
equivalent, with a formal proof in Coq proposed as the next
step~\cite[Sec.~6.4]{zhang2025thesis}.  Theorem~\ref{thm:surgery-do} gives the
conjecture a dependency-level formulation suitable for mechanization.

The same formulation distinguishes the dependency graph from a compatible
supplied graph, which may contain unused arrows.
Corollary~\ref{cor:supplied-graph} characterizes exactly when surgery on the
two graphs agrees.

\section{Conclusion}
\label{sec:conclusion}

Graph surgery and constant mechanism replacement agree after dependency
extraction:
\[
  \Graph(F^\iota)
  =
  \Surg\bigl(\Graph(F),T_\iota\bigr).
\]
For a graph supplied with the model, the analogous equality holds exactly
when surgery removes all of its unused arrows; it therefore holds for every
intervention if and only if the starting model is exact.  The construction
based on the supplied graph and the one based on the exact dependency graph
always have the same run; for an exact starting model, they are the same
intervened model.

The intervention equation characterizes the resulting outcome.  Sequential
interventions combine by retaining the last value assigned to each target, and
an outcome depends only on interventions at its actual dependency ancestors.
Together these results specify the \(\doexpr\)-operator directly while keeping
its structural and functional components distinct.

\subsubsection*{Acknowledgements.}
This work was supported by the Anusandhan National Research Foundation (ANRF),
Government of India, under the Prime Minister Early Career Research Grant
ANRF/ECRG/2025/001136/ENS.  I am grateful to Aalok Thakkar for helpful
discussions.

\bibliographystyle{splncs04}
\bibliography{references}

\end{document}